\documentclass[10pt]{article}
\usepackage{amsmath,amsfonts,amsthm,mathtools}
\usepackage[utf8]{inputenc}
\usepackage[T1]{fontenc}

\IfFileExists{newtxtext.sty}{
	\usepackage{newtxtext}
	\usepackage{newtxmath}
}{
	\IfFileExists{mathptmx.sty}{\usepackage{mathptmx}}{\usepackage{lmodern}}
}

\usepackage[margin=0.95in]{geometry}
\usepackage{microtype} 
\usepackage{booktabs,tabularx,array}
\usepackage{graphicx}
\usepackage[numbers,sort&compress]{natbib}
\usepackage{hyperref}
\usepackage{xcolor}
\usepackage{caption}
\usepackage{listings}
\usepackage{float}
\usepackage{tikz}
\usetikzlibrary{arrows.meta,positioning,calc} %

\hypersetup{
	colorlinks=true,
	linkcolor=blue!45!black,
	urlcolor=blue!45!black,
	citecolor=blue!45!black,
	hypertexnames=false
}

\newtheorem{definition}{Definition}[section]
\newtheorem{proposition}[definition]{Proposition}
\newtheorem{theorem}[definition]{Theorem}
\newtheorem{corollary}[definition]{Corollary}

\lstdefinelanguage{Lean}{
	morekeywords={import,namespace,end,def,theorem,lemma,variable,section,noncomputable,by,if,then,else,have,show,calc,exact,simpa,rw,let,forall},
	sensitive=true,
	morecomment=[l]--
}

\title{LPC-SM: Local Predictive Coding and Sparse Memory for Long-Context Language Modeling}

\author{
Keqin Xie\\
Independent Researcher, Suzhou, China\\
\href{mailto:xiekeqin30@gmail.com}{xiekeqin30@gmail.com}
}

\date{}

\begin{document}
\maketitle
\begin{abstract}
Most current long-context language models still rely on attention to handle both local interaction and long-range state, which leaves relatively little room to test alternative decompositions of sequence modeling. We propose LPC-SM, a hybrid autoregressive architecture that separates local attention, persistent memory, predictive correction, and run-time control within the same block, and we use Orthogonal Novelty Transport (ONT) to govern slow-memory writes. We evaluate a 158M-parameter model in three stages spanning base language modeling, mathematical continuation, and 4096-token continuation. Removing mHC raises the Stage-A final LM loss from 12.630 to 15.127, while adaptive sparse control improves the Stage-B final LM loss from 12.137 to 10.787 relative to a matched fixed-ratio continuation. The full route remains stable at sequence length 4096, where Stage C ends with final LM loss 11.582 and improves the delayed-identifier diagnostic from 14.396 to 12.031 in key cross-entropy. Taken together, these results show that long-context autoregressive modeling can be organized around a broader division of labor than attention alone.
\end{abstract}
\vspace{1em}

\section{Introduction}
Transformer language models have been scaled with striking success, and most of that success has come from making attention broader, denser, cheaper, or easier to reuse. Even when recurrence, compression, or retrieval enters the picture, attention usually remains the place where the model is expected to reconcile nearby context with far-away state. That default is strong enough that alternative decompositions of sequence modeling are often treated as peripheral unless they already outperform a mature Transformer baseline. We think that order of judgment is too restrictive. Before asking whether another decomposition wins, it is worth asking whether it can be made coherent, trainable, and empirically legible in its own right \citep{gu2024transformersSSM,de2024griffin,behrouz2025titans,dong2024hymba,ding2024longrope,xiao2024infllm,zhang2025kimiLinear}.

LPC-SM starts from that narrower question. We keep local causal attention for what it already does well: short-range precision. We then give longer-lived state to a dual-timescale memory, expose representation mismatch through an explicit predictive-correction pathway, and let a small set of learned controllers regulate sparsity, memory writing, and stopping behavior. The point is not to remove attention from the model, nor to argue that a recurrent state should replace it wholesale. The point is to test whether these roles become easier to study once they are assigned to different mechanisms rather than folded back into a single attention-dominant block.

That choice makes the slow-memory write unusually important. If chunk summaries repeatedly move in directions the slow state already represents, the model spends write capacity on reinforcement rather than accumulation. ONT is our answer to that problem. It leaves the component already aligned with the slow state untouched and amplifies only the orthogonal novelty component before the write. From one angle, this is a small geometric modification. From another, it is the moment where the architecture commits to the idea that memory should preserve what is already there and spend additional capacity on what is genuinely new.

We evaluate LPC-SM at 158M parameters in three stages: base language modeling, mathematical continuation, and a 4096-token continuation run. The evidence at this scale is uneven in a useful way. mHC and adaptive sparse control show clear gains. Slow memory helps, though more modestly. Predictive coding, ONT, and learned stopping are not yet well summarized by base LM loss alone. That asymmetry is informative. It suggests that the architecture is already doing enough for individual mechanisms to become separable, even if not all of them have reached the regime in which their intended benefits are fully visible.

\section{Related Work}
The immediate backdrop for LPC-SM is still the Transformer family, including sparse and local variants \citep{vaswani2017attention,child2019sparse,sukhbaatar2019adaptiveSpan,beltagy2020longformer,zaheer2020bigbird}. Those models differ sharply in efficiency, receptive field, and memory footprint, yet they share a common structural assumption: context is still mediated mainly through attention. That is the baseline from which LPC-SM departs. We are not asking whether attention can be made cheaper; we are asking what changes once attention is no longer asked to be the only durable carrier of sequence state.

A second thread in the literature weakens that assumption by introducing persistent state. Transformer-XL, Compressive Transformer, recurrent memory transformers, RetNet, RWKV, and Mamba all treat recurrence or compressed state as something more than a cache optimization \citep{dai2019transformerxl,rae2019compressive,bulatov2022rmt,sun2023retnet,peng2023rwkv,gu2023mamba}. More recent systems such as Griffin, Titans, Hymba, and Mamba-2 go further in blending attention with recurrent or state-space components \citep{de2024griffin,behrouz2025titans,dong2024hymba,gu2024transformersSSM}. LPC-SM is close to this line of work in spirit. The difference is that we make two distinctions explicit that are often left implicit: fast versus slow memory, and local correction versus long-range storage.

Long-context extension methods form a neighboring but distinct literature. LongRoPE and InfLLM show that existing models can often be pushed well beyond their nominal context range through positional extrapolation or memory-assisted inference \citep{ding2024longrope,xiao2024infllm}. We see those methods as evidence that long-context behavior is still pliable. At the same time, they leave the internal division of labor mostly intact. LPC-SM addresses a different question. Instead of stretching an attention-centered architecture outward, we reassign some of the work inward, at the block level, before asking how far the context can be extended.

The predictive-correction path draws on predictive-coding ideas in neuroscience and machine learning \citep{rao1999predictive,friston2005cortical,lotter2016deep,whittington2020predictive}. Most language models let depth absorb mismatch implicitly: hidden states are updated, but the disagreement between a local explanation and the current representation is not itself exposed as a first-class quantity. We chose to expose it because mismatch seems like exactly the kind of signal that should interact with internal control. Once that signal is available, connections to adaptive computation and routing become natural \citep{graves2016act,dehghani2018universal,elbayad2019depthAdaptive,shazeer2017moe,fedus2021switch}, though LPC-SM does not use those ideas in the usual token-skipping or expert-selection form.

\section{Model}
Let $x_{1:T}$ be a token sequence and let $h_t^0$ denote its token-plus-position embedding at position $t$. LPC-SM applies $L$ identical autoregressive blocks followed by a final normalization and two output heads. Figures~\ref{fig:lpcsm-overview}, \ref{fig:lpcsm-block}, and \ref{fig:lpcsm-ont} illustrate the overall stack, the internal structure of a single block, and the ONT write used by the slow-memory pathway.

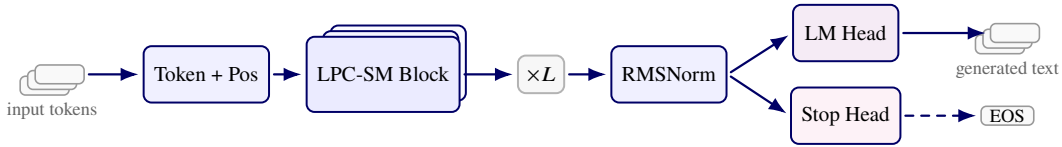
\begin{figure}[htbp]
	\centering
	\begin{tikzpicture}[
		>=Latex,
		line cap=round,
		line join=round,
		font=\footnotesize,
		module/.style={
			rounded corners=3pt,
			draw=blue!40!black, 
			line width=0.75pt,
			minimum height=0.75cm,
			align=center,
			inner sep=4pt
		},
		stack/.style={
			module,
			minimum height=1.00cm,
			minimum width=1.80cm
		},
		thinbox/.style={
			rounded corners=2pt,
			draw=black!40,
			line width=0.6pt,
			minimum height=0.25cm,
			minimum width=0.65cm,
			fill=gray!5,
			inner sep=1pt
		},
		flow/.style={
			->, 
			line width=0.85pt, 
			color=blue!40!black,
			shorten >=1.5pt, 
			shorten <=1.5pt
		}
		]
		
		\node[thinbox] (in1) at (-2.15, -0.15) {};
		\node[thinbox] (in2) at (-2.05, -0.075) {};
		\node[thinbox] (in3) at (-1.95, 0) {};
		\node[font=\scriptsize, text=gray!80!black, anchor=north] at (-2.05, -0.25) {input tokens};
		
		\node[module, fill=blue!6, minimum width=1.5cm] (tok) at (0, 0) {Token + Pos};
		
		\node[stack, fill=blue!2] at (2.55, 0.15) {};
		\node[stack, fill=blue!4] at (2.45, 0.075) {};
		\node[stack, fill=blue!8] (lpc) at (2.35, 0) {LPC-SM Block};
		
		\node[module, draw=black!30, fill=gray!6, minimum height=0.5cm, inner sep=3pt] (timesL) at (4.45, 0) {$\times L$};
		
		\node[module, fill=blue!6, minimum width=1.2cm] (norm) at (6.15, 0) {RMSNorm};
		
		\node[module, fill=violet!8, minimum width=1.4cm] (lmhead) at (8.5, 0.55) {LM Head};
		\node[module, fill=magenta!6, minimum width=1.4cm] (stophead) at (8.5, -0.55) {Stop Head};
		
		\node[thinbox] (out1) at (10.55, 0.40) {};
		\node[thinbox] (out2) at (10.65, 0.475) {};
		\node[thinbox] (out3) at (10.75, 0.55) {};
		\node[font=\scriptsize, text=gray!80!black, anchor=north] at (10.65, 0.30) {generated text};
		
		\node[thinbox, minimum width=0.7cm] (eos) at (10.65, -0.55) {\scriptsize EOS};
		
		\draw[flow] (in3) -- (tok);
		\draw[flow] (tok) -- (lpc);
		\draw[flow] (lpc) -- (timesL);
		\draw[flow] (timesL) -- (norm);
		
		\draw[flow] (norm.east) -- (lmhead.west);
		\draw[flow] (norm.east) -- (stophead.west);
		
		\draw[flow] (lmhead.east) -- (out3.west);
		\draw[flow, dashed] (stophead.east) -- (eos.west);
		
	\end{tikzpicture}
	\caption{Overall LPC-SM stack.}
	\label{fig:lpcsm-overview}
\end{figure}

\subsection{Block Structure}
Each block begins with RMSNorm \citep{zhang2019rmsnorm} and then combines three information sources at token position $t$: a local-attention read, a dual-timescale memory read, and a predictive correction. The local-attention path is windowed and causal,
\[
a_t = \mathrm{Attn}(h_{\max(1,t-w+1):t}),
\]
where $w$ is the local window. The goal of this path is local precision rather than long-range storage.

In the default configuration, queries, keys, and values are obtained from a shared linear projection. The implementation also supports a multi-head latent-attention variant in which
\[
q_t = W_q h_t,\qquad
z_t = W_z h_t,\qquad
k_t = W_k^\uparrow z_t,\qquad
v_t = W_v^\uparrow z_t,
\]
so that the key-value side is compressed through a latent bottleneck before being lifted back into head space. This option is not essential to the architecture, but it is part of the model family implemented in the code and allows us to vary how much of the local path is spent on explicit key-value bandwidth.

The memory path maintains a fast state updated every token and a slow state updated only at chunk boundaries. The fast state follows
\[
d_t = \sigma(W_d h_t), \qquad
u_t = \tanh(W_u h_t),
\]
\[
m_t^f = d_t \odot m_{t-1}^f + (1-d_t) \odot u_t.
\]
Separate gates query the fast and slow pathways,
\[
q_t^f = \sigma(W_{qf} h_t), \qquad q_t^s = \sigma(W_{qs} h_t),
\]
\[
r_t = W_r \big[ q_t^f \odot m_t^f \parallel q_t^s \odot m_{k-1}^s \big].
\]
This design gives the model a token-level recurrent trace and a chunk-level persistent state without forcing either one to replace local attention. 
The distinction matters because these two memories are not doing the same job at different timescales. The fast state remains close to tokenwise evidence, while the slow state only changes when a chunk has accumulated enough evidence to justify a write. Put differently, the architecture assumes that persistence should be selective, not continuous.

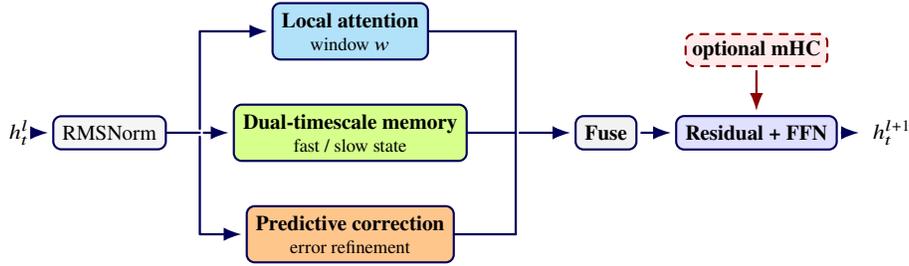
\begin{figure}[htbp]
	\centering
	\begin{tikzpicture}[
		>=Latex,
		line cap=round,
		line join=round,
		font=\footnotesize,
		wire/.style={line width=0.85pt, color=blue!35!black},
		flow/.style={->, line width=0.85pt, color=blue!35!black, shorten >=1.5pt, shorten <=1.5pt},
		module/.style={
			rounded corners=3pt,
			draw=blue!35!black,
			line width=0.8pt,
			align=center,
			inner sep=3.5pt
		},
		attn/.style={module, fill=cyan!28},
		mem/.style={module, fill=green!35!yellow!45},
		pred/.style={module, fill=orange!45},
		fuse/.style={module, fill=gray!10},
		res/.style={module, fill=blue!12},
		norm/.style={module, fill=gray!8},
		mhc/.style={module, dashed, draw=red!55!black, fill=red!8, inner sep=2.5pt}
		]
		
		\node (in) at (0, 0) {$h_t^l$};
		\node[norm] (norm) at (1.2, 0) {RMSNorm};
		
		\node[attn] (attn) at (4.4, 1.35) {\textbf{Local attention}\\[-0.1em]\scriptsize window $w$};
		\node[mem] (mem) at (4.4, 0) {\textbf{Dual-timescale memory}\\[-0.1em]\scriptsize fast / slow state};
		\node[pred] (pred) at (4.4, -1.35) {\textbf{Predictive correction}\\[-0.1em]\scriptsize error refinement};
		
		\node[fuse] (fuse) at (7.8, 0) {\textbf{Fuse}};
		\node[res] (res) at (9.8, 0) {\textbf{Residual + FFN}};
		\node[mhc] (mhc) at (9.8, 1.1) {\textbf{optional mHC}};
		\node (out) at (11.6, 0) {$h_t^{l+1}$};
		
		\draw[flow] (in) -- (norm);
		
		\coordinate (split) at (2.4, 0);
		\draw[wire] (norm) -- (split);
		
		\draw[flow] (split) |- (attn);
		\draw[flow] (split) -- (mem);
		\draw[flow] (split) |- (pred);
		
		\coordinate (mergeTop) at (6.6, 1.35);
		\coordinate (mergeBot) at (6.6, -1.35);
		\coordinate (mergeMid) at (6.6, 0);
		
		\draw[wire] (mergeTop) -- (mergeBot);
		
		\draw[wire] (attn) -- (mergeTop);
		\draw[wire] (mem) -- (mergeMid);
		\draw[wire] (pred) -- (mergeBot);
		
		\draw[flow] (mergeMid) -- (fuse);
		
		\draw[flow] (fuse) -- (res);
		\draw[->, line width=0.85pt, color=red!55!black, shorten >=1.5pt, shorten <=1.5pt] (mhc) -- (res);
		\draw[flow] (res) -- (out);
		
	\end{tikzpicture}
	\caption{A single LPC-SM block.}
	\label{fig:lpcsm-block}
\end{figure}

\subsection{Dual-Timescale Memory and ONT}
At the end of chunk $C_k$, the block forms a chunk summary
\[
c_k = \frac{1}{|C_k|}\sum_{t \in C_k} m_t^f.
\]
The slow-memory gate is input dependent,
\[
g_k = \sigma(W_g h_t).
\]
The question is how $c_k$ should be transported before it is written into the slow state. ONT defines the aligned component relative to the \textit{previous} slow state $m_{k-1}^s$ by
\[
\mathrm{proj}(c_k \mid m_{k-1}^s)
=
\begin{cases}
	0, & \text{if } \|m_{k-1}^s\|_2 = 0,\\[6pt]
	\dfrac{c_k^\top m_{k-1}^s}{\|m_{k-1}^s\|_2^2} m_{k-1}^s, & \text{otherwise,}
\end{cases}
\]
and the novelty component by
\[
n_k = c_k - \mathrm{proj}(c_k \mid m_{k-1}^s).
\]
The transported summary is
\[
c_k^\star = c_k + \alpha_n n_k,
\]
where $\alpha_n \ge 0$ is the novelty coefficient. The slow-memory update then follows:
\[
\tilde{u}_k = \tanh(W_c c_k^\star),
\qquad
m_k^s = g_k \odot m_{k-1}^s + (1-g_k) \odot \tilde{u}_k.
\]
During autoregressive generation, the same write rule operates on a per-layer cache that keeps the attention history, fast state, slow state, and the running partial chunk summary. That choice is easy to overlook, but it matters. A slow-memory mechanism can look appealing on paper and still drift into a train-inference mismatch if prompt-side partial chunks are treated differently from training chunks. We therefore keep the write order aligned across the two regimes as closely as possible.

\begin{figure}[htbp]
	\centering
	\begin{tikzpicture}[
		>=Latex,
		line cap=round,
		line join=round,
		font=\footnotesize,
		box/.style={
			rounded corners=3pt,
			draw=blue!35!black,
			line width=0.8pt,
			align=center,
			inner sep=3.5pt 
		},
		chunk/.style={box, fill=yellow!28},
		state/.style={box, fill=gray!8},
		op/.style={box, fill=gray!8},
		proj/.style={box, fill=blue!12},
		nov/.style={box, fill=green!18},
		transport/.style={box, fill=gray!8}
		]

		\node[chunk] (N1) at (-1.8, 2.6) {$c_k$\\[-0.1em]\scriptsize chunk summary};
		\node[state] (N2) at (1.8, 2.6) {$m_k^s$\\[-0.1em]\scriptsize slow state};
		
		\node[op] (N3) at (0, 1.4) {\textbf{decompose}\\[-0.1em]\scriptsize along $m_k^s$};
		
		\node[proj] (N4) at (-1.5, 0.2) {$P_m(c_k)$\\[-0.1em]\scriptsize aligned};
		\node[nov] (N5) at (1.5, 0.2) {$N_m(c_k)$\\[-0.1em]\scriptsize novelty};
		
		\node[transport] (N6) at (0, -0.8) {$c_k^\star = P + (1+\alpha_n)N$};
		
		\node[state] (N7) at (0, -1.7) {\textbf{slow write}};

		\begin{scope}[->, line width=0.8pt, color=blue!35!black, shorten >=1.5pt, shorten <=1.5pt]
			\draw (N1) -- (N3);
			\draw (N2) -- (N3);
			\draw (N3) -- (N4);
			\draw (N3) -- (N5);
			\draw (N4) -- (N6);
			\draw (N5) -- (N6);
			\draw (N6) -- (N7);
		\end{scope}
		
	\end{tikzpicture}
	\caption{Orthogonal Novelty Transport for slow-memory writes.}
	\label{fig:lpcsm-ont}
\end{figure}
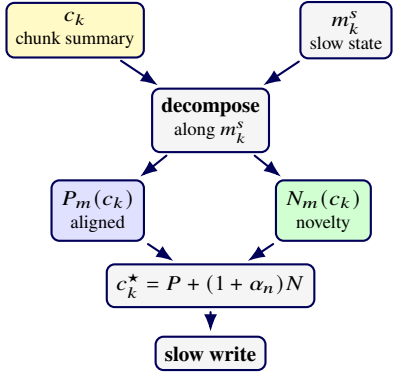

\subsection{Correction and Stopping}
The block also predicts the current hidden state from local context and memory and then corrects that prediction with an explicit mismatch signal. 

The predictor is initialized by
\[
\hat{h}_t^{(0)} = f_{\mathrm{pred}} \big( [ a_t \parallel r_t ] \big),
\]
and refined iteratively:
\[
\hat{h}_t^{(s+1)} = \hat{h}_t^{(s)} + f_{\mathrm{refine}} \big( [ a_t \parallel r_t \parallel h_t - \hat{h}_t^{(s)} ] \big).
\]

LPC-SM does not force this pathway to act uniformly at every token. Instead, a learned controller converts error statistics into a sparse event mask through score normalization, a learned bias-scale transform, a temperature, and a straight-through hard threshold. The sparse ratio itself remains learnable within prescribed bounds. That detail is central to how we interpret the controller: the architecture is not merely sparsified; it is allowed to choose how sparse it wants to be within a bounded regime.

The same model family also contains an optional multi-head-coupled residual router (mHC), following the hyper-connection view developed in \citep{xie2025mhc}. Our use is narrower than the full formulation in that work: here mHC sits inside each LPC-SM block as a residual transport layer rather than as a global system-level redesign. Rather than forwarding a single hidden stream directly through the block, mHC lifts the state into multiple streams, learns pre-mixing weights, applies a Sinkhorn-normalized residual transport across streams, and then injects the updated block output back through learned post-mixing coefficients. What we found empirically is that this is not a cosmetic addition. At 158M parameters, mHC is the mechanism whose removal hurts most. That makes it more natural to read mHC as part of the core block geometry than as an optional embellishment.

\subsection{Training Objective}
The training objective combines next-token prediction with auxiliary terms for predictive correction, sparsity, memory magnitude, and stopping:
\[
\begin{aligned}
L
={}&
L_{\mathrm{lm}}
+
\lambda_{\mathrm{pred}} L_{\mathrm{pred}}
+
\lambda_{\mathrm{sparse}} L_{\mathrm{sparse}} \\
&+
\lambda_{\mathrm{mem}} L_{\mathrm{mem}}
+
\lambda_{\mathrm{stop}} L_{\mathrm{stop}}.
\end{aligned}
\]
The language-model term is standard cross-entropy. The auxiliary terms are there for a narrower reason: they keep the explicit mechanisms from becoming inert. Predictive correction is penalized through its mismatch signal, sparsity is regularized so the controller cannot drift arbitrarily, memory magnitude is constrained to keep the recurrent path from dominating by scale alone, and the stop head is nudged toward EOS-sensitive behavior. We do not claim that this objective is the only reasonable one. We claim something smaller: once the architecture exposes correction, sparsity, memory, and stopping as explicit state variables, it is natural to let the training objective acknowledge them rather than pretend they are invisible.

\section{Experimental Setup}
\subsection{Staged Training Program}
All experiments reported here use the same 158,313,241-parameter model with a GPT-2 tokenizer. Training proceeds in three stages summarized in Table~\ref{tab:stages}.

\begin{table}[htbp]
\centering
\small
\caption{Staged training schedule for the 158M study.}
\label{tab:stages}
\begin{tabular}{lrrr}
\toprule
Stage & Corpus & Tokens & Seq. \\
\midrule
A & Dolma3-base & 32.77M & 2048 \\
B & OpenWebMath-10k & 16.38M & 2048 \\
C & LongMino continuation & 24.58M & 4096 \\
\bottomrule
\end{tabular}
\end{table}

We evaluate LPC-SM across base modeling, mathematical continuation, and longer-context continuation under a common training pipeline. The staged schedule is not only a convenience for limited compute. It also separates three questions that would be entangled in a single monolithic run: whether the architecture can serve as a base language model at all, whether its internal control signals remain useful when the domain shifts toward mathematics, and whether the same route survives a substantial increase in effective context length. Keeping these stages explicit makes the later comparisons easier to interpret.

Given the parameter count of 158M, the token budgets in Stages A and B (32.77M and 16.38M, respectively) place the model in a significant underfitting regime relative to standard scaling laws \citep{kaplan2020scaling,hoffmann2022chinchilla}. We treat these runs as a proof-of-concept study of structural emergence, numerical stability, and functional interaction among the LPC-SM modules rather than as a compute-optimal perplexity target. The small-scale setting is useful precisely because it makes the early behavior of the controllers and the ONT pathway easier to inspect.

\subsection{Experimental Design}
Stage A includes the full model together with five ablations: slow memory removed, predictive coding removed, ONT disabled, stop head removed, and mHC removed. This stage is designed to answer a narrow architectural question: which mechanisms affect the optimization behavior of the base model under a fixed parameter budget? Stage B compares adaptive sparse control against a fixed sparse-ratio control initialized from the same Stage-A checkpoint and trained on the same data for the same budget. Because the initialization and continuation corpus are matched, the comparison isolates the value of learned control more cleanly than a comparison across independently trained models would. Stage C extends the full route to sequence length 4096. Here the emphasis is not on winning against another method, but on whether the full architecture remains trainable once longer-sequence recurrence, chunked writes, and explicit correction are exercised together.

\subsection{Evaluation Criteria}
We use final LM loss at the end of each stage as the primary metric. We also report training throughput, the learned sparse ratio, and fixed-prompt sanity checks. Final LM loss is not intended to summarize every design goal of LPC-SM. In particular, predictive correction, ONT, and the stop head are meant to affect behavior that is only partially visible through base pretraining loss. We nevertheless treat final loss as the most stable common measure across all reported runs, and we interpret the ablations with that limitation in mind rather than reading every gain or loss as a full verdict on the underlying mechanism.

\section{Results}
Tables~\ref{tab:stage-a}, \ref{tab:continuation}, and \ref{tab:passkey} summarize the 158M study.

\subsection{Stage-A Ablations}
\begin{table}[htbp]
\centering
\small
\caption{Stage-A ablations at 158M parameters. Lower is better.}
\label{tab:stage-a}
\begin{tabular}{lrrrr}
\toprule
Variant & Final LM & $\Delta$ (\%) & Tok/s & Final ratio \\
\midrule
Full LPC-SM & 12.630 & 0.000 & 6798 & 0.226 \\
w/o slow memory & 12.671 & +0.320 & 21938 & 0.249 \\
w/o predictive coding & 12.413 & -1.719 & 7083 & 0.600 \\
w/o ONT & 11.781 & -6.724 & 6676 & 0.235 \\
w/o stop head & 12.078 & -4.377 & 6630 & 0.228 \\
w/o mHC & 15.127 & +19.764 & 7038 & 0.234 \\
\bottomrule
\end{tabular}
\end{table}

Stage A separates the block into mechanisms that affect optimization in visibly different ways. The clearest regression comes from removing mHC: the final LM loss rises from 12.630 to 15.127, which is large enough to treat residual routing as part of the effective core block rather than as an optional refinement. Removing slow memory changes the final loss only slightly, but the direction is still unfavorable to the ablation. That is a weaker signal than the mHC result, yet it is consistent with the claim that the recurrent path is doing some useful work even before the model has been trained at larger scale.

The remaining ablations are less straightforward. Removing predictive coding, ONT, or the stop head lowers the base-stage LM loss. In the present regime, we interpret that result cautiously. The model is strongly undertrained relative to its parameter count, and several LPC-SM components are not designed primarily to improve short-budget next-token loss. Predictive correction, novelty-constrained transport, and stopping all bias the internal organization of the model toward behaviors whose payoff is more likely to appear under continuation, longer-range conditioning, or downstream adaptation than in the most immediate Stage-A metric.

\subsection{Continuation Results}
\begin{table}[htbp]
\centering
\small
\caption{Continuation runs at 158M parameters.}
\label{tab:continuation}
\begin{tabular}{llrrrrr}
\toprule
Stage & Variant & Seq. & Final LM & $\Delta$ (\%) & Tok/s & Final ratio \\
\midrule
B & Adaptive sparse control & 2048 & 10.787 & 0.000 & 6728 & 0.214 \\
B & Fixed sparse control & 2048 & 12.137 & +12.517 & 6711 & 0.226 \\
C & Adaptive long-context continuation & 4096 & 11.582 & -- & 3711 & 0.215 \\
\bottomrule
\end{tabular}
\end{table}

Stage B provides the clearest evidence that the internal controller is doing substantive work rather than merely tracking training noise. The adaptive run improves final LM loss by 12.5\% relative to the matched fixed-ratio control while holding initialization, data, and training budget constant. Because the two runs differ only in whether the sparse ratio remains learnable, the comparison isolates the role of adaptive control more cleanly than the broader Stage-A ablations. In this setting, allowing the controller to move appears to help the model rebalance computation as the continuation domain shifts from general text to mathematics.

Stage C addresses a different question. Here the issue is not whether the controller beats a fixed alternative, but whether the full route remains trainable once the sequence length doubles from 2048 to 4096. The run completes without removing the memory pathway, predictive correction, routing, or learned control, and it finishes with final LM loss 11.582. From the perspective of architecture validation, this matters because a hybrid block can look reasonable at short sequence lengths yet become brittle once recurrence, chunked writes, and longer continuation interact. That behavior does not appear in the present run.

\begin{table}[htbp]
\centering
\small
\caption{Diagnostic delayed-identifier probe. Lower key cross-entropy is better. The probe uses a header that introduces a key, a long distractor region, and a trigger prefix that asks the model to continue the identifier. Stage-A values use six 2040-token prompts; the Stage-C column scores the same probe after long-context continuation.}
\label{tab:passkey}
\begin{tabular}{lrrrr}
\toprule
Probe & Stage-A full & w/o slow memory & w/o ONT & Stage-C full \\
\midrule
Delayed identifier CE & 14.396 & 13.865 & 15.427 & 12.031 \\
\bottomrule
\end{tabular}
\end{table}

Table~\ref{tab:passkey} adds a more diagnostic view of long-range conditioning than free generation does at this stage. Instead of asking the model to produce an answer from scratch, we score the cross-entropy of the correct delayed identifier under teacher forcing. Two patterns stand out. First, disabling ONT worsens the diagnostic relative to the full model, which is consistent with the idea that novelty-aware writes help preserve delayed information. Second, the full model improves substantially after Stage C continuation, with the delayed-identifier cross-entropy falling from 14.396 to 12.031 on the same probe family. The comparison with the slow-memory ablation remains mixed: at 158M, the no-memory variant is still slightly better on this probe than the full Stage-A model. For that reason, we do not read Table~\ref{tab:passkey} as a definitive isolation of the slow-memory mechanism. We read it as evidence that long-context continuation sharpens delayed conditioning in the full route, while the contribution of individual memory components is not yet cleanly separated at the present scale.

\section{Conclusion}
We introduced LPC-SM as a long-context autoregressive architecture that separates local attention, persistent memory, predictive correction, and internal control within the same block. Across the 158M study, the strongest empirical support comes from mHC and adaptive sparse control. Removing mHC causes the largest Stage-A regression, while adaptive sparse control clearly outperforms a matched fixed-ratio continuation in Stage B. The full route also remains stable when the sequence length doubles to 4096 in Stage C.

The evidence for the memory pathway is more qualified. Slow memory is compatible with stable long-context continuation, and the delayed-identifier probe improves materially after Stage C training, with the key cross-entropy falling from 14.396 to 12.031. At the same time, the 158M ablations do not yet isolate a uniformly positive effect for every memory-related component under every metric. That is the main reason we treat the present paper as an architecture validation study rather than as a claim of compute-optimal superiority.

Within that scope, the results are still meaningful. They show that the LPC-SM decomposition can be trained end to end, that several of its internal control mechanisms matter measurably, and that longer-context continuation sharpens delayed conditioning in the full model. Larger 1B-scale runs are currently in progress.

\appendix

\section{Mathematical Properties of ONT}
We state here the mathematical properties corresponding to the ONT write used in the implementation. The underlying geometry is the standard orthogonal decomposition of a vector into components parallel and orthogonal to a reference direction \citep{meyer2000matrix}, specialized here to the slow-memory write rule used by LPC-SM.

\begin{definition}[ONT projection, novelty, and transport]
Let $E$ be a real inner-product space. For $c, m \in E$ and $\alpha \in \mathbb{R}$, define
\[
P_m(c) =
\begin{cases}
0, & \text{if } m = 0,\\[4pt]
\dfrac{\langle c, m \rangle}{\|m\|^2} m, & \text{if } m \neq 0,
\end{cases}
\qquad
N_m(c) = c - P_m(c),
\]
and
\[
T_\alpha(c,m) = c + \alpha N_m(c).
\]
\end{definition}

\begin{definition}[Comparison target and feasible set]
For $c \in E$ and $\alpha \in \mathbb{R}$, define the comparison target
\[
Y_\alpha(c) = (1+\alpha)c.
\]
For $c, m \in E$, define the feasible affine set
\[
\mathcal{A}(c,m) = \{x \in E : \langle x, m \rangle = \langle c, m \rangle \}.
\]
\end{definition}

\begin{proposition}[Basic decomposition and aligned gap]
\label{prop:ont-basic}
For every $c, m \in E$ and $\alpha \in \mathbb{R}$,
\[
c = P_m(c) + N_m(c),
\qquad
\langle N_m(c), m \rangle = 0,
\]
and
\[
T_\alpha(c,m) = P_m(c) + (1+\alpha)N_m(c).
\]
Moreover,
\[
\begin{aligned}
Y_\alpha(c) &= T_\alpha(c,m) + \alpha P_m(c),\\
T_\alpha(c,m) - Y_\alpha(c) &= -\alpha P_m(c).
\end{aligned}
\]
\end{proposition}

\begin{proof}
The identity $c = P_m(c) + N_m(c)$ is immediate from the definition of $N_m(c)$. When $m=0$, the orthogonality claim reduces to $\langle c,0\rangle=0$. When $m \neq 0$, $P_m(c)$ is the usual projection of $c$ onto the one-dimensional subspace spanned by $m$, so
\[
\langle N_m(c),m\rangle
=
\langle c - P_m(c),m\rangle
=
\langle c,m\rangle - \frac{\langle c,m\rangle}{\|m\|^2}\langle m,m\rangle
=0.
\]
Substituting $c = P_m(c) + N_m(c)$ into $T_\alpha(c,m)=c+\alpha N_m(c)$ gives
\[
T_\alpha(c,m)=P_m(c)+(1+\alpha)N_m(c).
\]
Likewise,
\[
Y_\alpha(c)
=(1+\alpha)\bigl(P_m(c)+N_m(c)\bigr)
=T_\alpha(c,m)+\alpha P_m(c),
\]
and the final identity follows by subtraction.
\end{proof}

\begin{proposition}[Feasibility]
\label{prop:ont-feasible}
For every $c,m \in E$ and $\alpha \in \mathbb{R}$, the ONT transport satisfies
\[
\langle T_\alpha(c,m), m \rangle = \langle c, m \rangle.
\]
Equivalently, $T_\alpha(c,m) \in \mathcal{A}(c,m)$.
\end{proposition}

\begin{proof}
By definition,
\[
\langle T_\alpha(c,m),m\rangle
=
\langle c + \alpha N_m(c),m\rangle
=
\langle c,m\rangle + \alpha \langle N_m(c),m\rangle.
\]
The orthogonality result from Proposition~\ref{prop:ont-basic} implies $\langle N_m(c),m\rangle=0$, hence
\[
\langle T_\alpha(c,m),m\rangle=\langle c,m\rangle.
\]
\end{proof}

\begin{theorem}[ONT is the constrained minimizer]
\label{thm:ont-minimal}
Let $E$ be a real inner-product space, let $c,m \in E$, and let $\alpha \in \mathbb{R}$. Then for every $x \in \mathcal{A}(c,m)$,
\[
\|x - Y_\alpha(c)\|^2
=
\|x - T_\alpha(c,m)\|^2
+
\|T_\alpha(c,m) - Y_\alpha(c)\|^2.
\]
Consequently,
\[
\|T_\alpha(c,m) - Y_\alpha(c)\|
\leq
\|x - Y_\alpha(c)\|
\qquad
\text{for all } x \in \mathcal{A}(c,m).
\]
Thus $T_\alpha(c,m)$ is a minimizer of
\[
\min\{\|x - Y_\alpha(c)\| : x \in \mathcal{A}(c,m)\}.
\]
\end{theorem}

\begin{proof}
Fix $x \in \mathcal{A}(c,m)$. Since both $x$ and $T_\alpha(c,m)$ lie in $\mathcal{A}(c,m)$, we have
\[
\langle x - T_\alpha(c,m), m \rangle = 0.
\]
By Proposition~\ref{prop:ont-basic}, $P_m(c)$ is either zero or a scalar multiple of $m$, hence
\[
\langle x - T_\alpha(c,m), P_m(c) \rangle = 0.
\]
Using the identity
\[
T_\alpha(c,m) - Y_\alpha(c) = -\alpha P_m(c),
\]
we obtain
\[
\langle x - T_\alpha(c,m), T_\alpha(c,m) - Y_\alpha(c) \rangle = 0.
\]
Now decompose
\[
x - Y_\alpha(c)
=
\bigl(x - T_\alpha(c,m)\bigr)
+
\bigl(T_\alpha(c,m) - Y_\alpha(c)\bigr).
\]
The two summands are orthogonal, so the Pythagorean theorem yields
\[
\|x - Y_\alpha(c)\|^2
=
\|x - T_\alpha(c,m)\|^2
+
\|T_\alpha(c,m) - Y_\alpha(c)\|^2.
\]
Since the first term on the right-hand side is nonnegative, the minimality inequality follows immediately.
\end{proof}

\begin{corollary}[Uniqueness]
\label{cor:ont-unique}
Under the assumptions of Theorem~\ref{thm:ont-minimal}, $T_\alpha(c,m)$ is the unique minimizer of
\[
\min\{\|x - Y_\alpha(c)\| : x \in \mathcal{A}(c,m)\}.
\]
\end{corollary}

\begin{proof}
Suppose $x \in \mathcal{A}(c,m)$ is also a minimizer. Then
\[
\|x - Y_\alpha(c)\|
=
\|T_\alpha(c,m) - Y_\alpha(c)\|.
\]
Substituting this equality into the identity from Theorem~\ref{thm:ont-minimal} gives
\[
\|x - T_\alpha(c,m)\|^2 = 0.
\]
Hence $\|x - T_\alpha(c,m)\|=0$, so $x=T_\alpha(c,m)$.
\end{proof}

\begin{corollary}[Hilbert-space generalization]
\label{cor:ont-hilbert}
Let $H$ be a real Hilbert space. For every $c,m \in H$ and every $\alpha \in \mathbb{R}$, the ONT transport $T_\alpha(c,m)$ is the unique minimizer of
\[
\min\{\|x - (1+\alpha)c\| : x \in H,\ \langle x,m\rangle = \langle c,m\rangle\}.
\]
\end{corollary}

\begin{proof}
The preceding arguments use only the axioms of a real inner-product space and therefore already hold in every such space. A real Hilbert space is, by definition, a complete real inner-product space, so the result follows by specialization.
\end{proof}

\subsection{Variational Characterization of ONT}
The ONT update is the write rule used by LPC-SM to insert a chunk summary into slow memory. The corresponding design problem is to construct a write $x$ that preserves the component already aligned with the current slow-memory state $m$, while favoring motion in the novelty direction $N_m(c)$. For fixed $c,m \in E$ and $\alpha \in \mathbb{R}$, consider
\[
\begin{aligned}
\min_{x \in \mathcal{A}(c,m)}\quad &\mathcal{J}_{\alpha,c,m}(x),\\
\mathcal{J}_{\alpha,c,m}(x)
={}&
\frac{1}{2}\|x-c\|^2
- \alpha \langle x-c, N_m(c)\rangle,
\end{aligned}
\]
where $\mathcal{A}(c,m)=\{x \in E : \langle x,m\rangle = \langle c,m\rangle\}$ \citep{cheney1959proximity,bauschke1996projection,deutsch1992alternating,bauschke2021bestApprox,bauschke2015projectionMethods}. In the implementation, $\alpha=\alpha_n \ge 0$; the theorem is stated for arbitrary $\alpha \in \mathbb{R}$ because the characterization itself does not require a sign restriction.

\begin{theorem}[ONT solves the slow-memory write problem]
\label{thm:ont-variational}
Let $E$ be a real inner-product space. For every $c,m \in E$ and every $\alpha \in \mathbb{R}$, the ONT transport $T_\alpha(c,m)$ is the unique minimizer of
\[
\min\{\mathcal{J}_{\alpha,c,m}(x) : x \in \mathcal{A}(c,m)\}.
\]
\end{theorem}

\begin{proof}
Write $N = N_m(c)$ and $T = T_\alpha(c,m)=c+\alpha N$. By Proposition~\ref{prop:ont-feasible}, $T \in \mathcal{A}(c,m)$. Fix any feasible write $x \in \mathcal{A}(c,m)$. Then
\[
\begin{aligned}
\mathcal{J}_{\alpha,c,m}(x)
&=
\frac{1}{2}\|x-c\|^2 - \alpha \langle x-c,N\rangle \\
&=
\frac{1}{2}\|(x-c)-\alpha N\|^2 - \frac{1}{2}\alpha^2\|N\|^2.
\end{aligned}
\]
Since $(x-c)-\alpha N = x - (c+\alpha N)=x-T$, we obtain
\[
\mathcal{J}_{\alpha,c,m}(x)
=
\frac{1}{2}\|x-T\|^2 - \frac{1}{2}\alpha^2\|N\|^2.
\]
Evaluating the same identity at $x=T$ gives
\[
\mathcal{J}_{\alpha,c,m}(T)
=
- \frac{1}{2}\alpha^2\|N\|^2.
\]
Therefore, for every feasible $x$,
\[
\mathcal{J}_{\alpha,c,m}(x) - \mathcal{J}_{\alpha,c,m}(T)
=
\frac{1}{2}\|x-T\|^2
\ge 0.
\]
Hence $T$ minimizes $\mathcal{J}_{\alpha,c,m}$ over $\mathcal{A}(c,m)$. If $x$ is any other minimizer, then the display above implies $\|x-T\|^2=0$, and therefore $x=T$. Thus the minimizer is unique.
\end{proof}

\section{Formal Development}
\begin{lstlisting}[language=Lean,caption={Module aggregator.}]
import LPCSMFormal.ONT
import LPCSMFormal.ONTOptimality
\end{lstlisting}

\begin{lstlisting}[language=Lean,caption={Core ONT geometry.}]
import Mathlib.Analysis.InnerProductSpace.Basic
import Mathlib.Tactic

noncomputable section

namespace LPCSMFormal

variable {E : Type*} [NormedAddCommGroup E] [InnerProductSpace Real E]

def rinner (x y : E) : Real :=
inner Real x y

def ontProj (chunkSummary slowMemory : E) : E :=
if _h : norm slowMemory = 0 then
0
else
SMul.smul ((rinner chunkSummary slowMemory) / (norm slowMemory ^ 2)) slowMemory

def ontNovelty (chunkSummary slowMemory : E) : E :=
chunkSummary - ontProj chunkSummary slowMemory

def ontTransport (alpha : Real) (chunkSummary slowMemory : E) : E :=
chunkSummary + SMul.smul alpha (ontNovelty chunkSummary slowMemory)

theorem ontProj_zero_right (chunkSummary : E) :
ontProj chunkSummary 0 = 0 := by
simp [ontProj]

theorem ontDecomposition (chunkSummary slowMemory : E) :
ontProj chunkSummary slowMemory + ontNovelty chunkSummary slowMemory = chunkSummary := by
simp [ontNovelty]

theorem ontNovelty_inner_slowMemory (chunkSummary slowMemory : E) :
inner Real (ontNovelty chunkSummary slowMemory) slowMemory = 0 := by
by_cases h : norm slowMemory = 0
case pos =>
have hs : slowMemory = 0 := by
exact norm_eq_zero.mp h
simp [ontNovelty, ontProj, hs]
case neg =>
have hpos : 0 < norm slowMemory := by
exact lt_of_le_of_ne (norm_nonneg slowMemory) (Ne.symm h)
have hnorm : Ne (norm slowMemory ^ 2) 0 := by
positivity
rw [ontNovelty, inner_sub_left]
simp [ontProj, h]
have hinner :
inner Real (SMul.smul (rinner chunkSummary slowMemory / (norm slowMemory ^ 2)) slowMemory) slowMemory =
(rinner chunkSummary slowMemory / (norm slowMemory ^ 2)) * inner Real slowMemory slowMemory := by
simpa using
(real_inner_smul_left
slowMemory
slowMemory
(rinner chunkSummary slowMemory / (norm slowMemory ^ 2)))
rw [hinner, real_inner_self_eq_norm_sq]
field_simp [hnorm]
simp [rinner]

theorem ontTransport_eq_proj_plus_scaled_novelty (alpha : Real) (chunkSummary slowMemory : E) :
ontTransport alpha chunkSummary slowMemory =
ontProj chunkSummary slowMemory +
SMul.smul (1 + alpha) (ontNovelty chunkSummary slowMemory) := by
calc
ontTransport alpha chunkSummary slowMemory
= chunkSummary + SMul.smul alpha (ontNovelty chunkSummary slowMemory) := by
rfl
_ = ontProj chunkSummary slowMemory
+ ontNovelty chunkSummary slowMemory
+ SMul.smul alpha (ontNovelty chunkSummary slowMemory) := by
rw [ontDecomposition]
_ = ontProj chunkSummary slowMemory
+ (ontNovelty chunkSummary slowMemory
+ SMul.smul alpha (ontNovelty chunkSummary slowMemory)) := by
simp [add_assoc]
_ = ontProj chunkSummary slowMemory
+ SMul.smul (1 + alpha) (ontNovelty chunkSummary slowMemory) := by
have hsmul :
ontNovelty chunkSummary slowMemory
+ SMul.smul alpha (ontNovelty chunkSummary slowMemory) =
SMul.smul (1 + alpha) (ontNovelty chunkSummary slowMemory) := by
simpa [one_smul] using (add_smul 1 alpha (ontNovelty chunkSummary slowMemory)).symm
rw [hsmul]

theorem ontTransport_inner_slowMemory (alpha : Real) (chunkSummary slowMemory : E) :
rinner (ontTransport alpha chunkSummary slowMemory) slowMemory =
rinner chunkSummary slowMemory := by
rw [ontTransport, rinner, inner_add_left]
have hsmul :
inner Real (SMul.smul alpha (ontNovelty chunkSummary slowMemory)) slowMemory =
alpha * inner Real (ontNovelty chunkSummary slowMemory) slowMemory := by
simpa using real_inner_smul_left (ontNovelty chunkSummary slowMemory) slowMemory alpha
rw [hsmul, ontNovelty_inner_slowMemory]
simp [rinner]

end LPCSMFormal

\end{lstlisting}

\begin{lstlisting}[language=Lean,caption={Optimality, uniqueness, and variational characterization.}]
import LPCSMFormal.ONT
import Mathlib.Analysis.InnerProductSpace.Basic
import Mathlib.Tactic

noncomputable section

namespace LPCSMFormal

variable {E : Type*} [NormedAddCommGroup E] [InnerProductSpace Real E]

-- Derived analytical objects used to characterize the code-defined transport.
-- The implementation primitive remains `ontTransport` from `ONT.lean`.
def ontTarget (alpha : Real) (chunkSummary : E) : E :=
SMul.smul (1 + alpha) chunkSummary

def ontWriteObjective (alpha : Real) (chunkSummary slowMemory x : E) : Real :=
(1 / 2 : Real) * (norm (x - chunkSummary) ^ 2)
- alpha * inner Real (x - chunkSummary) (ontNovelty chunkSummary slowMemory)

theorem ontTarget_eq_transport_add_scaled_proj (alpha : Real) (chunkSummary slowMemory : E) :
ontTarget alpha chunkSummary =
ontTransport alpha chunkSummary slowMemory + SMul.smul alpha (ontProj chunkSummary slowMemory) := by
unfold ontTarget
conv_lhs =>
rw [<- ontDecomposition chunkSummary slowMemory]
have hsmuladd :
SMul.smul (1 + alpha) (ontProj chunkSummary slowMemory + ontNovelty chunkSummary slowMemory) =
SMul.smul (1 + alpha) (ontProj chunkSummary slowMemory)
+ SMul.smul (1 + alpha) (ontNovelty chunkSummary slowMemory) := by
exact
smul_add
(1 + alpha)
(ontProj chunkSummary slowMemory)
(ontNovelty chunkSummary slowMemory)
rw [hsmuladd]
rw [show SMul.smul (1 + alpha) (ontProj chunkSummary slowMemory) =
ontProj chunkSummary slowMemory + SMul.smul alpha (ontProj chunkSummary slowMemory) by
simpa [one_smul] using (add_smul 1 alpha (ontProj chunkSummary slowMemory))]
rw [ontTransport_eq_proj_plus_scaled_novelty]
simp [add_assoc, add_left_comm, add_comm]

theorem ontTransport_sub_target_eq_neg_scaled_proj (alpha : Real) (chunkSummary slowMemory : E) :
ontTransport alpha chunkSummary slowMemory - ontTarget alpha chunkSummary =
SMul.smul (-alpha) (ontProj chunkSummary slowMemory) := by
calc
ontTransport alpha chunkSummary slowMemory - ontTarget alpha chunkSummary
= ontTransport alpha chunkSummary slowMemory
- (ontTransport alpha chunkSummary slowMemory + SMul.smul alpha (ontProj chunkSummary slowMemory)) := by
rw [ontTarget_eq_transport_add_scaled_proj alpha chunkSummary slowMemory]
_ = -SMul.smul alpha (ontProj chunkSummary slowMemory) := by
abel_nf
_ = SMul.smul (-alpha) (ontProj chunkSummary slowMemory) := by
exact (neg_smul alpha (ontProj chunkSummary slowMemory)).symm

theorem feasible_inner_slow_eq_zero
(alpha : Real)
(chunkSummary slowMemory x : E)
(hfeas : rinner x slowMemory = rinner chunkSummary slowMemory) :
inner Real (x - ontTransport alpha chunkSummary slowMemory) slowMemory = 0 := by
have htransport :
inner Real (ontTransport alpha chunkSummary slowMemory) slowMemory =
inner Real chunkSummary slowMemory := by
simpa [rinner] using (ontTransport_inner_slowMemory alpha chunkSummary slowMemory)
rw [rinner] at hfeas
rw [inner_sub_left]
rw [hfeas, htransport]
simp [rinner]

theorem feasible_inner_proj_eq_zero
(alpha : Real)
(chunkSummary slowMemory x : E)
(hfeas : rinner x slowMemory = rinner chunkSummary slowMemory) :
inner Real (x - ontTransport alpha chunkSummary slowMemory) (ontProj chunkSummary slowMemory) = 0 := by
by_cases h : norm slowMemory = 0
case pos =>
have hs : slowMemory = 0 := by
exact norm_eq_zero.mp h
simp [ontProj, hs]
case neg =>
have hproj :
ontProj chunkSummary slowMemory =
SMul.smul ((rinner chunkSummary slowMemory) / (norm slowMemory ^ 2)) slowMemory := by
simp [ontProj, h]
have hinner :
inner Real (x - ontTransport alpha chunkSummary slowMemory)
(SMul.smul ((rinner chunkSummary slowMemory) / (norm slowMemory ^ 2)) slowMemory) =
((rinner chunkSummary slowMemory) / (norm slowMemory ^ 2))
* inner Real (x - ontTransport alpha chunkSummary slowMemory) slowMemory := by
simpa using
real_inner_smul_right
(x - ontTransport alpha chunkSummary slowMemory)
slowMemory
((rinner chunkSummary slowMemory) / (norm slowMemory ^ 2))
rw [hproj, hinner, feasible_inner_slow_eq_zero alpha chunkSummary slowMemory x hfeas]
simp

theorem feasible_inner_transport_target_eq_zero
(alpha : Real)
(chunkSummary slowMemory x : E)
(hfeas : rinner x slowMemory = rinner chunkSummary slowMemory) :
inner Real
(x - ontTransport alpha chunkSummary slowMemory)
(ontTransport alpha chunkSummary slowMemory - ontTarget alpha chunkSummary) = 0 := by
rw [ontTransport_sub_target_eq_neg_scaled_proj]
have hneg :
SMul.smul (-alpha) (ontProj chunkSummary slowMemory) =
-SMul.smul alpha (ontProj chunkSummary slowMemory) := by
exact neg_smul alpha (ontProj chunkSummary slowMemory)
rw [hneg, inner_neg_right]
have hinner :
inner Real
(x - ontTransport alpha chunkSummary slowMemory)
(SMul.smul alpha (ontProj chunkSummary slowMemory)) =
alpha * inner Real
(x - ontTransport alpha chunkSummary slowMemory)
(ontProj chunkSummary slowMemory) := by
simpa using
real_inner_smul_right
(x - ontTransport alpha chunkSummary slowMemory)
(ontProj chunkSummary slowMemory)
alpha
rw [hinner, feasible_inner_proj_eq_zero alpha chunkSummary slowMemory x hfeas]
simp

theorem ontTransport_objective_decomposition
(alpha : Real)
(chunkSummary slowMemory x : E)
(hfeas : rinner x slowMemory = rinner chunkSummary slowMemory) :
norm (x - ontTarget alpha chunkSummary) * norm (x - ontTarget alpha chunkSummary) =
norm (x - ontTransport alpha chunkSummary slowMemory) * norm (x - ontTransport alpha chunkSummary slowMemory)
+ norm (ontTransport alpha chunkSummary slowMemory - ontTarget alpha chunkSummary)
* norm (ontTransport alpha chunkSummary slowMemory - ontTarget alpha chunkSummary) := by
have hsplit :
x - ontTarget alpha chunkSummary =
(x - ontTransport alpha chunkSummary slowMemory)
+ (ontTransport alpha chunkSummary slowMemory - ontTarget alpha chunkSummary) := by
abel_nf
have horth :
inner Real
(x - ontTransport alpha chunkSummary slowMemory)
(ontTransport alpha chunkSummary slowMemory - ontTarget alpha chunkSummary) = 0 := by
exact feasible_inner_transport_target_eq_zero alpha chunkSummary slowMemory x hfeas
rw [hsplit]
exact
norm_add_sq_eq_norm_sq_add_norm_sq_of_inner_eq_zero
(x - ontTransport alpha chunkSummary slowMemory)
(ontTransport alpha chunkSummary slowMemory - ontTarget alpha chunkSummary)
horth

theorem ontTransport_minimal
(alpha : Real)
(chunkSummary slowMemory x : E)
(hfeas : rinner x slowMemory = rinner chunkSummary slowMemory) :
norm (ontTransport alpha chunkSummary slowMemory - ontTarget alpha chunkSummary) <=
norm (x - ontTarget alpha chunkSummary) := by
have hdecomp :=
ontTransport_objective_decomposition alpha chunkSummary slowMemory x hfeas
have hnonneg :
0 <= norm (x - ontTransport alpha chunkSummary slowMemory)
* norm (x - ontTransport alpha chunkSummary slowMemory) := by
positivity
have hsq :
norm (ontTransport alpha chunkSummary slowMemory - ontTarget alpha chunkSummary)
* norm (ontTransport alpha chunkSummary slowMemory - ontTarget alpha chunkSummary)
<= norm (x - ontTarget alpha chunkSummary) * norm (x - ontTarget alpha chunkSummary) := by
nlinarith [hdecomp]
have hleft : 0 <= norm (ontTransport alpha chunkSummary slowMemory - ontTarget alpha chunkSummary) := by
exact norm_nonneg _
have hright : 0 <= norm (x - ontTarget alpha chunkSummary) := by
exact norm_nonneg _
nlinarith

theorem ontTransport_unique_minimizer
(alpha : Real)
(chunkSummary slowMemory x : E)
(hfeas : rinner x slowMemory = rinner chunkSummary slowMemory)
(hmin :
forall y : E,
rinner y slowMemory = rinner chunkSummary slowMemory ->
norm (x - ontTarget alpha chunkSummary) <= norm (y - ontTarget alpha chunkSummary)) :
x = ontTransport alpha chunkSummary slowMemory := by
have hx_le :
norm (x - ontTarget alpha chunkSummary) <=
norm (ontTransport alpha chunkSummary slowMemory - ontTarget alpha chunkSummary) := by
apply hmin
exact ontTransport_inner_slowMemory alpha chunkSummary slowMemory
have hopt_le :
norm (ontTransport alpha chunkSummary slowMemory - ontTarget alpha chunkSummary) <=
norm (x - ontTarget alpha chunkSummary) := by
exact ontTransport_minimal alpha chunkSummary slowMemory x hfeas
have heq_obj :
norm (x - ontTarget alpha chunkSummary) =
norm (ontTransport alpha chunkSummary slowMemory - ontTarget alpha chunkSummary) := by
exact le_antisymm hx_le hopt_le
have hdecomp :=
ontTransport_objective_decomposition alpha chunkSummary slowMemory x hfeas
have hzero_sq :
norm (x - ontTransport alpha chunkSummary slowMemory)
* norm (x - ontTransport alpha chunkSummary slowMemory) = 0 := by
rw [heq_obj] at hdecomp
nlinarith [hdecomp]
have hzero_norm :
norm (x - ontTransport alpha chunkSummary slowMemory) = 0 := by
have hnonneg : 0 <= norm (x - ontTransport alpha chunkSummary slowMemory) := by
exact norm_nonneg _
nlinarith
exact sub_eq_zero.mp (norm_eq_zero.mp hzero_norm)

theorem ontWriteObjective_eq_square_completion
(alpha : Real)
(chunkSummary slowMemory x : E) :
ontWriteObjective alpha chunkSummary slowMemory x =
(1 / 2 : Real) * (norm (x - ontTransport alpha chunkSummary slowMemory) ^ 2)
- (1 / 2 : Real) * (alpha ^ 2) * (norm (ontNovelty chunkSummary slowMemory) ^ 2) := by
set u : E := x - chunkSummary
set n : E := ontNovelty chunkSummary slowMemory
have h_expand :
inner Real (u - SMul.smul alpha n) (u - SMul.smul alpha n) =
inner Real u u - 2 * alpha * inner Real u n + alpha ^ 2 * inner Real n n := by
have h1 :
inner Real (SMul.smul alpha n) u = alpha * inner Real n u := by
simpa using real_inner_smul_left n u alpha
have h2 :
inner Real u (SMul.smul alpha n) = alpha * inner Real u n := by
simpa using real_inner_smul_right u n alpha
have h3 :
inner Real (SMul.smul alpha n) (SMul.smul alpha n) = alpha ^ 2 * inner Real n n := by
have h3' :
inner Real n (SMul.smul alpha n) = alpha * inner Real n n := by
simpa using real_inner_smul_right n n alpha
calc
inner Real (SMul.smul alpha n) (SMul.smul alpha n)
= alpha * inner Real n (SMul.smul alpha n) := by
exact real_inner_smul_left n (SMul.smul alpha n) alpha
_ = alpha * (alpha * inner Real n n) := by
rw [h3']
_ = alpha ^ 2 * inner Real n n := by
ring
have h4 :
inner Real (SMul.smul alpha n) (u - SMul.smul alpha n) =
inner Real (SMul.smul alpha n) u - inner Real (SMul.smul alpha n) (SMul.smul alpha n) := by
rw [inner_sub_right]
rw [inner_sub_left, inner_sub_right]
rw [h2, h4, h1, h3]
have hcomm : inner Real n u = inner Real u n := by
simpa using (real_inner_comm n u).symm
rw [hcomm]
ring
have h_rearr :
(1 / 2 : Real) * inner Real u u - alpha * inner Real u n =
(1 / 2 : Real) * inner Real (u - SMul.smul alpha n) (u - SMul.smul alpha n)
- (1 / 2 : Real) * (alpha ^ 2) * inner Real n n := by
nlinarith [h_expand]
have hu :
u - SMul.smul alpha n = x - ontTransport alpha chunkSummary slowMemory := by
simp [u, n, ontTransport]
abel_nf
calc
ontWriteObjective alpha chunkSummary slowMemory x
= (1 / 2 : Real) * inner Real u u - alpha * inner Real u n := by
unfold ontWriteObjective
rw [show inner Real (x - chunkSummary) (ontNovelty chunkSummary slowMemory) = inner Real u n by rfl]
rw [show inner Real u u = norm u ^ 2 by simp]
_ = (1 / 2 : Real) * inner Real (u - SMul.smul alpha n) (u - SMul.smul alpha n)
- (1 / 2 : Real) * (alpha ^ 2) * inner Real n n := by
exact h_rearr
_ = (1 / 2 : Real) * (norm (u - SMul.smul alpha n) ^ 2)
- (1 / 2 : Real) * (alpha ^ 2) * (norm n ^ 2) := by
rw [show inner Real (u - SMul.smul alpha n) (u - SMul.smul alpha n) =
norm (u - SMul.smul alpha n) ^ 2 by
simp]
rw [show inner Real n n = norm n ^ 2 by simp]
_ = (1 / 2 : Real) * (norm (x - ontTransport alpha chunkSummary slowMemory) ^ 2)
- (1 / 2 : Real) * (alpha ^ 2) * (norm (ontNovelty chunkSummary slowMemory) ^ 2) := by
rw [hu]

theorem ontWriteObjective_minimal
(alpha : Real)
(chunkSummary slowMemory x : E) :
ontWriteObjective alpha chunkSummary slowMemory (ontTransport alpha chunkSummary slowMemory) <=
ontWriteObjective alpha chunkSummary slowMemory x := by
rw [ontWriteObjective_eq_square_completion, ontWriteObjective_eq_square_completion]
have hopt_zero :
norm
(ontTransport alpha chunkSummary slowMemory - ontTransport alpha chunkSummary slowMemory) ^ 2 = 0 := by
simp
have hnonneg : 0 <= (1 / 2 : Real) * (norm (x - ontTransport alpha chunkSummary slowMemory) ^ 2) := by
positivity
nlinarith [hopt_zero]

theorem ontWriteObjective_unique_minimizer
(alpha : Real)
(chunkSummary slowMemory x : E)
(hmin :
forall y : E,
ontWriteObjective alpha chunkSummary slowMemory x <=
ontWriteObjective alpha chunkSummary slowMemory y) :
x = ontTransport alpha chunkSummary slowMemory := by
have hx_le :
ontWriteObjective alpha chunkSummary slowMemory x <=
ontWriteObjective alpha chunkSummary slowMemory (ontTransport alpha chunkSummary slowMemory) := by
exact hmin (ontTransport alpha chunkSummary slowMemory)
have hopt_le :
ontWriteObjective alpha chunkSummary slowMemory (ontTransport alpha chunkSummary slowMemory) <=
ontWriteObjective alpha chunkSummary slowMemory x := by
exact ontWriteObjective_minimal alpha chunkSummary slowMemory x
have heq_obj :
ontWriteObjective alpha chunkSummary slowMemory x =
ontWriteObjective alpha chunkSummary slowMemory (ontTransport alpha chunkSummary slowMemory) := by
exact le_antisymm hx_le hopt_le
have hzero_sq :
norm (x - ontTransport alpha chunkSummary slowMemory) ^ 2 = 0 := by
have heq_obj' := heq_obj
rw [ontWriteObjective_eq_square_completion, ontWriteObjective_eq_square_completion] at heq_obj'
have hopt_zero :
norm
(ontTransport alpha chunkSummary slowMemory - ontTransport alpha chunkSummary slowMemory) ^ 2 = 0 := by
simp
nlinarith [heq_obj', hopt_zero]
have hzero_norm : norm (x - ontTransport alpha chunkSummary slowMemory) = 0 := by
have hnonneg : 0 <= norm (x - ontTransport alpha chunkSummary slowMemory) := by
exact norm_nonneg _
nlinarith
exact sub_eq_zero.mp (norm_eq_zero.mp hzero_norm)

theorem ontWriteObjective_minimal_feasible
(alpha : Real)
(chunkSummary slowMemory x : E)
(_hfeas : rinner x slowMemory = rinner chunkSummary slowMemory) :
ontWriteObjective alpha chunkSummary slowMemory (ontTransport alpha chunkSummary slowMemory) <=
ontWriteObjective alpha chunkSummary slowMemory x := by
exact ontWriteObjective_minimal alpha chunkSummary slowMemory x

theorem ontWriteObjective_unique_minimizer_feasible
(alpha : Real)
(chunkSummary slowMemory x : E)
(hfeas : rinner x slowMemory = rinner chunkSummary slowMemory)
(hmin :
forall y : E,
rinner y slowMemory = rinner chunkSummary slowMemory ->
ontWriteObjective alpha chunkSummary slowMemory x <=
ontWriteObjective alpha chunkSummary slowMemory y) :
x = ontTransport alpha chunkSummary slowMemory := by
have hx_le :
ontWriteObjective alpha chunkSummary slowMemory x <=
ontWriteObjective alpha chunkSummary slowMemory (ontTransport alpha chunkSummary slowMemory) := by
apply hmin
exact ontTransport_inner_slowMemory alpha chunkSummary slowMemory
have hopt_le :
ontWriteObjective alpha chunkSummary slowMemory (ontTransport alpha chunkSummary slowMemory) <=
ontWriteObjective alpha chunkSummary slowMemory x := by
exact ontWriteObjective_minimal_feasible alpha chunkSummary slowMemory x hfeas
have heq_obj :
ontWriteObjective alpha chunkSummary slowMemory x =
ontWriteObjective alpha chunkSummary slowMemory (ontTransport alpha chunkSummary slowMemory) := by
exact le_antisymm hx_le hopt_le
have hzero_sq :
norm (x - ontTransport alpha chunkSummary slowMemory) ^ 2 = 0 := by
have heq_obj' := heq_obj
rw [ontWriteObjective_eq_square_completion, ontWriteObjective_eq_square_completion] at heq_obj'
have hopt_zero :
norm
(ontTransport alpha chunkSummary slowMemory - ontTransport alpha chunkSummary slowMemory) ^ 2 = 0 := by
simp
nlinarith [heq_obj', hopt_zero]
have hzero_norm : norm (x - ontTransport alpha chunkSummary slowMemory) = 0 := by
have hnonneg : 0 <= norm (x - ontTransport alpha chunkSummary slowMemory) := by
exact norm_nonneg _
nlinarith
exact sub_eq_zero.mp (norm_eq_zero.mp hzero_norm)

section Hilbert

variable {H : Type*} [NormedAddCommGroup H] [InnerProductSpace Real H] [CompleteSpace H]

theorem ontTransport_minimal_hilbert
(alpha : Real)
(chunkSummary slowMemory x : H)
(hfeas : rinner x slowMemory = rinner chunkSummary slowMemory) :
norm (ontTransport alpha chunkSummary slowMemory - ontTarget alpha chunkSummary) <=
norm (x - ontTarget alpha chunkSummary) := by
let _ := (inferInstance : CompleteSpace H)
exact ontTransport_minimal alpha chunkSummary slowMemory x hfeas

theorem ontTransport_unique_minimizer_hilbert
(alpha : Real)
(chunkSummary slowMemory x : H)
(hfeas : rinner x slowMemory = rinner chunkSummary slowMemory)
(hmin :
forall y : H,
rinner y slowMemory = rinner chunkSummary slowMemory ->
norm (x - ontTarget alpha chunkSummary) <= norm (y - ontTarget alpha chunkSummary)) :
x = ontTransport alpha chunkSummary slowMemory := by
let _ := (inferInstance : CompleteSpace H)
exact ontTransport_unique_minimizer alpha chunkSummary slowMemory x hfeas hmin

theorem ontWriteObjective_minimal_feasible_hilbert
(alpha : Real)
(chunkSummary slowMemory x : H)
(hfeas : rinner x slowMemory = rinner chunkSummary slowMemory) :
ontWriteObjective alpha chunkSummary slowMemory (ontTransport alpha chunkSummary slowMemory) <=
ontWriteObjective alpha chunkSummary slowMemory x := by
let _ := (inferInstance : CompleteSpace H)
exact ontWriteObjective_minimal_feasible alpha chunkSummary slowMemory x hfeas

theorem ontWriteObjective_unique_minimizer_feasible_hilbert
(alpha : Real)
(chunkSummary slowMemory x : H)
(hfeas : rinner x slowMemory = rinner chunkSummary slowMemory)
(hmin :
forall y : H,
rinner y slowMemory = rinner chunkSummary slowMemory ->
ontWriteObjective alpha chunkSummary slowMemory x <=
ontWriteObjective alpha chunkSummary slowMemory y) :
x = ontTransport alpha chunkSummary slowMemory := by
let _ := (inferInstance : CompleteSpace H)
exact ontWriteObjective_unique_minimizer_feasible alpha chunkSummary slowMemory x hfeas hmin

end Hilbert

end LPCSMFormal

\end{lstlisting}

\bibliographystyle{unsrtnat}
\bibliography{references}

\end{document}